\documentclass[journal]{IEEEtran}

\usepackage{cite}
\usepackage{amsthm}
\usepackage{amsmath,amssymb,amsfonts}
\usepackage{algorithmic}
\usepackage{algorithm}
\usepackage{graphicx}
\usepackage{booktabs}
\usepackage{multirow}
\usepackage{array}
\usepackage{tabularx}
\usepackage{subcaption}
\usepackage{orcidlink}
\usepackage{url}
\usepackage{xcolor}
\usepackage[normalem]{ulem}
\usepackage{tikz}
\usepackage{pgfplots}
\pgfplotsset{compat=1.18}
\usetikzlibrary{positioning,arrows.meta,fit,backgrounds,calc}
\newcommand{\revise}[1]{\textcolor{black}{#1}}
\newtheorem{proposition}{Proposition}

\begin{document}

\title{SpikingNav: Robust Embodied Navigation with Spiking Neural Policies}

\author{Jiahong~Zhang \orcidlink{0000-0002-5687-1839}\IEEEauthorrefmark{2},
        Sijun~Shen \orcidlink{0009-0009-0654-6358}\IEEEauthorrefmark{2},
        Dehua~Wu \orcidlink{0000-0003-1743-1790},
        Yifan~Lin \orcidlink{0009-0007-9915-039X},
        Xuechen~Xia \orcidlink{0009-0004-0587-3502},
        % Dehua~Wu ,
        Xu~Chu \orcidlink{0009-0003-4659-5599}, 
        Youhui~Zhang \orcidlink{0000-0003-2333-3580},
        % Han Xu \orcidlink{0000-0003-1411-3092},
        % Yonghong~Tian,~\IEEEmembership{Fellow,~IEEE},
        % Bo~Xu \orcidlink{0000-0002-1111-1529},        
        and~Guoqi~Li \orcidlink{0000-0002-8994-431X}\IEEEauthorrefmark{1}%

\vspace{-1em}
\thanks{\IEEEauthorrefmark{2}Jiahong Zhang and Sijun Shen contributed equally to this work.}%
\thanks{Jiahong Zhang, Yifan Lin, Xuechen Xia, and Guoqi Li are with the Institute of Automation, Chinese Academy of Sciences, Beijing 100045, China, and with the School of Artificial Intelligence, University of Chinese Academy of Sciences, Beijing 100049, China.}%
\thanks{Sijun Shen is with the State Key Laboratory of Media Convergence and Communication, Communication University of China, Beijing 100024, China.}%
% \thanks{Yonghong Tian is with Peking University, Beijing 100871, China, and with Peng Cheng Laboratory, Shenzhen 518066, China.}%
\thanks{ Dehua Wu, Xu Chu and Youhui Zhang are with the Department of Computer Science and Technology, Tsinghua University, Beijing 100084, China. }
\thanks{Youhui Zhang is also with the Beijing National Research Center for Information Science and Technology, Tsinghua University, Beijing 100084, China.}
\thanks{\IEEEauthorrefmark{1}Corresponding author: Guoqi Li.}}

\markboth{IEEE Transactions on Industrial Cyber-Physical Systems,~Vol.~XX, No.~XX, Month~2026}%
{Zhang \MakeLowercase{\textit{et al.}}: SpikingNav: Robust Embodied Navigation with Spiking Neural Policies}

\maketitle

\begin{abstract}
Embodied navigation requires an agent to make sequential decisions from {egocentric observations} \revise{in a physical environment}. Existing Artificial Neural Network (ANN)-based navigation models have achieved strong performance, yet they often rely on dense computation and may degrade under visual corruptions. Spiking neural networks (SNNs) provide event-driven computation and intrinsic temporal dynamics, which are promising for compact and robust navigation \revise{on resource-constrained platforms}. However, whether spike-based sensing and policy dynamics can improve robustness in visually rich embodied navigation remains an open problem. This paper proposes \textbf{SpikingNav}, a spiking framework for robust indoor embodied navigation. It contains a Spiking Sensing Encoder (SSE) and a Spiking Policy Network (SPN). The SSE extracts task-conditioned visual features with a spike-based backbone. The SPN maintains a recurrent policy state through membrane integration, thresholding, and spike-triggered reset. In this way, SpikingNav exploits the dynamic properties and spike activations of SNNs to improve navigation performance and robustness.
We evaluate SpikingNav on PointNav and ObjectNav under clean observations and visual corruptions. SpikingNav achieves competitive clean performance and stronger robustness with fewer parameters and lower per-step computation than a matched ANN baseline. For instance, SpikingNav improves ObjectNav success from 31.05\% to 34.12\%, and raises the average success under visual corruptions from {8.45\% to 13.71\%}, demonstrating the benefits of spike-based sensing and policy dynamics. 
We further validate the deployability of our spike-based sensing method on the Thruster--V2 neuromorphic chip. \revise{This physical hardware validation shows that SpikingNav can be instantiated on a real neuromorphic substrate for cyber-physical systems.}
\end{abstract}


\section{Introduction}

Embodied navigation is a core capability for \revise{industrial cyber-physical agents} that must perceive, reason, and act in visually complex environments. Unlike static visual recognition and detection problems, navigation requires continuous decision making from streaming observations, where the agent must remain effective under sensor noise and scene changes~\cite{savva2019habitat,wu2024embodied}. These requirements make embodied navigation not only a perception problem, but also a sequential control problem for \revise{cyber-physical} systems deployed in the physical world. 

Recent progress in embodied AI has been largely driven by artificial neural networks (ANNs). Strong visual backbones and reinforcement learning (RL) frameworks have enabled impressive performance on tasks such as PointNav and ObjectNav \cite{chaplot2020object, batra2020objectnav}.
However, this progress also exposes an important practical limitation: high navigation performance often depends on dense feature extraction and computationally heavy policy networks, which are not always suitable for compact embodied platforms~\cite{nwm,zhang2025embodied,gao2026drive}. More importantly, embodied agents can be sensitive to corruptions and deployment-time disturbances, such as blur and noise~\cite{robustnav}. As a result, improving robustness under realistic perturbations while keeping computation affordable remains an open problem.

Spiking neural networks (SNNs) provide a promising direction for addressing this challenge, as they combine event-driven computation with intrinsic temporal dynamics~\cite{maass1997networks,stbp,fang2023spikingjelly}. These properties suggest possible advantages for embodied navigation, where both efficiency and temporally robust decision-making are important. Despite recent progress in SNN architectures and training methods~\cite{wu2019direct,fang2021sew,yao2024spikformer,mpsn}, however, most strong empirical results have still been reported on frame-level perception tasks~\cite{su2023deep,ren2026language}. In such settings, SNNs often exhibit appealing properties such as robustness and low-energy event-driven computation, yet these advantages do not consistently translate into superior performance over mature ANN backbones because of spike discretization and optimization difficulty~\cite{sharmin2020inherent,wu2024rsc}. It remains unclear whether these trade-offs persist in embodied navigation, where the challenge lies not only in semantic discrimination but also in stable temporal information integration under varying sensory inputs.

To investigate this problem, we present a spiking visual navigation framework, termed \textbf{SpikingNav}, which extends SNNs from perceptual encoding to the navigation policy core \revise{for compact physical navigation systems}. SpikingNav consists of two spike-based components. The \textbf{Spiking Sensing Encoder (SSE)} extracts visual representations with a spike-based ResNet-style backbone, encodes the task-specific goal representation, and fuses the two streams into a compact task-conditioned feature. The \textbf{Spiking Policy Network (SPN)} then updates the policy representation through membrane-based spiking recurrent dynamics before producing the actor and critic outputs. This design is motivated by the sequential nature of navigation, where reliable action selection depends not only on instantaneous visual evidence but also on temporally consistent state estimation. In this context, membrane leakage, thresholding, and reset of spiking neurons help suppress transient visual fluctuations before they propagate into the policy state.

Compared with prior SNN-based navigation studies mainly focused on mapless control, spatial coding, or hierarchical motion generation~\cite{yang2023memorynav,chai2025braininspired,yang2025hsrl}, our study considers a broader and more challenging setting by jointly evaluating PointNav and ObjectNav in visually rich indoor environments~\cite{deitke2020robothor}. Within this setting, we examine whether an SNN-based navigator can achieve competitive performance relative to an ANN navigator of similar model scale, and whether the spiking policy degrades more gracefully under visual corruptions. \revise{We further connect this algorithmic evaluation to a real neuromorphic processor on Thruster--V2\textsuperscript{\ref{foot:gaban_name}}, thereby grounding the proposed navigation pipeline in a physical computing system.}

Extensive experiments show that SpikingNav achieves high performance under clean settings while exhibiting stronger robustness. Compared with ANNNav, SpikingNav uses fewer parameters and lower per-step FLOPs (12.1M/0.97G vs. 14.0M/4.21G), improves clean ObjectNav success rate from {31.05\% to 34.12\%}, and {raises the average corrupted ObjectNav success rate from 8.45\% to 13.71\%}. 
Taken together, these results suggest that the practical value of SNNs in embodied systems lies in enabling compact and disturbance-tolerant sensing and policy modules for navigation. 
% We further conduct a module-level deployment of the SSE on the GaBAN V2 neuromorphic chip. The on-chip execution characteristics of spike-based visual perception and target fusion are evaluated in terms of execution cycles, dynamic current, and dynamic energy. 
The main contributions of this work are summarized as follows:

\begin{enumerate}
\item We propose SpikingNav, a spiking navigation framework that integrates an SSE and an SPN into a standard actor-critic RL pipeline, extending SNNs from perceptual encoding to policy learning for visually rich indoor navigation.

\item We develop the SPN as a compact recurrent policy core that converts sensing features into membrane states and updates them through spiking neural dynamics. By regulating the policy state through spike-triggered reset, the SPN supports stable action generation under corrupted observations.

\item Experimental results show that SpikingNav improves robustness under visual corruptions while maintaining competitive clean-task performance and lower model cost than the ANN baseline.

\item We implement the SSE on the Thruster--V2 neuromorphic chip and evaluate its on-chip execution characteristics, including cycles, dynamic current, and dynamic energy, providing an initial hardware-level validation of spike-based visual perception and target fusion \revise{on a physical neuromorphic computing system}.

\end{enumerate}

\section{Related Work}

\subsection{Spiking Neural Networks}

SNNs are widely regarded as the third generation of neural models due to their utilization of temporally structured spikes rather than continuous activations~\cite{maass1997networks}. Despite this biological plausibility, SNNs traditionally face optimization challenges owing to the non-differentiable nature of discrete activations, often resulting in a performance gap compared to Artificial Neural Networks (ANNs). Modern advances have sought to bridge this gap, driving SNNs toward ANN-level performance through innovations such as spatio-temporal backpropagation~\cite{stbp}, SLAYER~\cite{shrestha2018slayer}, direct training~\cite{wu2019direct}, and specialized normalization techniques~\cite{tdbn,tebn}. Furthermore, the development of deeper and more expressive architectures, including SEW-ResNet~\cite{fang2021sew}, MS-ResNet~\cite{hu2024advancing}, and Spikformer~\cite{yao2024spikformer,yao2025scaling}, has significantly expanded the scaling potential of spiking models. A parallel line of research investigates the intrinsic robustness and temporal modeling capabilities of SNNs, positing that spike discretization and membrane dynamics can enhance noise tolerance and sequential representation~\cite{park2021noise,rathi2020robust,kim2022rate,mukhoty2023certified}.

Despite these theoretical and architectural milestones, the majority of empirical evidence for SNNs remains confined to feedforward vision tasks, event-based classification, or object detection~\cite{su2023deep,luo2024integer}. Recent research involving spiking models also primarily targets frame-level perception rather than the real-time state estimation and action generation required for active agents~\cite{ren2026language}. Even in studies where robustness is a central theme, the prevailing benchmarks are typically restricted to visual recognition or sequence classification rather than the more complex domain of navigation policy learning~\cite{zhang2019fast,zheng2018sparse,mpsn,shen2026stage}. Our work studies SNNs within an autonomous navigation framework and demonstrates their effectiveness relative to ANNs, without requiring modifications to the standard RL task interface.

\begin{figure*}[t]
  \centering
  \includegraphics[width=\linewidth]{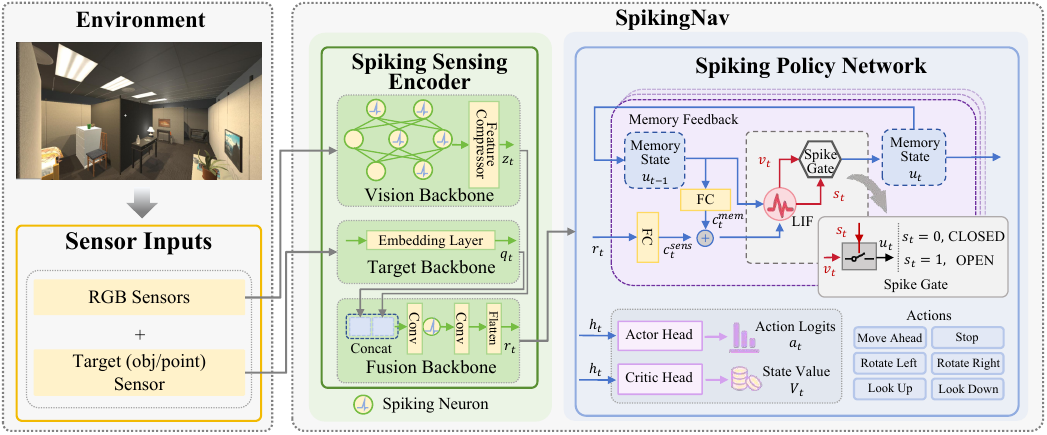}
  \caption{Overview of SpikingNav. Given the current RGB observation and a target, the Spiking Sensing Encoder (SSE) first extracts observation features through a visual backbone and maps the target into a target feature via an embedding layer. A lightweight fusion backbone then aligns and compresses the two representations to produce sensing features. The resulting feature is fed into the Spiking Policy Network (SPN), which maintains spike-driven temporal dynamics and outputs the action policy and state value while preserving the standard actor-critic RL interface.}
  \label{fig:framework}
\end{figure*}

\subsection{Embodied Visual Navigation}

Visual navigation has evolved from classical mapping-and-planning pipelines to end-to-end learned policies that process RGB, depth, and language-conditioned observations~\cite{wijmans2020ddppo,allenact2020,wu2024vln,yu2025mossvln}. Among embodied benchmarks, PointNav mainly evaluates goal-directed movement, whereas ObjectNav additionally requires semantic grounding and is therefore more sensitive to perception failures~\cite{deitke2020robothor}. Recent navigation systems typically rely on CNN or Transformer visual encoders, coupled with recurrent or hierarchical memory to integrate past observations and actions for sequential decision-making~\cite{pirlnav,ovon,mcfc,zhang2026beyond}.

SNN-based navigation has also received increasing attention~\cite{oikonomou2025reinforcement}. 
Existing studies have explored RL with spiking policies for low-level mobile robot navigation~\cite{10802416,10786920}, demonstrating the practical potential of spike-based control in mapless scenarios~\cite{9002834,yang2023memorynav,10670281}. 
Brain-inspired navigation further introduces spatial encoders based on grid cells, head-direction cells, and boundary-vector cells~\cite{chai2025braininspired}. 
Hierarchical spiking control combines a spiking high-level controller with continuous attractor dynamics to generate physically feasible robot actions~\cite{yang2025hsrl}.

Prior work suggests that spiking dynamics may support navigation by improving memory, spatial representation, and action generation. However, current evidence is mainly limited to relatively simple or mapless settings with lidar-based or other low-dimensional sensory inputs. As a result, their effectiveness in visually grounded embodied navigation under complex scene semantics remains unclear. In this work, we study SNN-based navigation in RoboTHOR~\cite{deitke2020robothor}, comparing a matched SNN-ANN pair under both clean and corrupted conditions across PointNav and ObjectNav within a unified policy framework.

\section{Method}

\subsection{Spiking Neural Network Preliminaries}

SNNs replace the continuous-valued activations in ANNs with temporally evolving neuron states and discrete spikes. A typical spiking neuron integrates its input into a membrane potential and emits a spike only when the potential crosses a threshold~\cite{maass1997networks,stbp}. This formulation can be written as:
\begin{align}
\widetilde{U}_t &= \lambda U_{t-1} + I_t, \\
S_t &= \mathcal{H}(\widetilde{U}_t - \vartheta), \\
U_t &= \widetilde{U}_t \odot (1-S_t),
\end{align}
where $\widetilde{U}_t$ is the pre-reset membrane potential, $U_t$ is the membrane potential after reset, $I_t$ is the input current, and \(S_t\in\{0,1\}^{d}\) denotes the spike output at internal spiking step \(t\), where \(t\in\{1,\ldots,T\}\). Here, \(T\) is the total number of internal spiking steps.
% where $\widetilde{U}_t$ is the pre-reset membrane potential, $U_t$ is the membrane potential after reset, $I_t$ is the input current, and $S_t\in\{0,1\}^{d}$ denotes the spike output at internal spiking step $t$. 
$\lambda$ is a leak factor, $\odot$ denotes element-wise multiplication, $\vartheta$ is the firing threshold, and $\mathcal{H}(\cdot)$ denotes the spike-generation function.
Because $\mathcal{H}(\cdot)$ is non-differentiable, modern SNNs are usually trained with surrogate-gradient methods that approximate its derivative during backpropagation~\cite{stbp}.

Compared with conventional ANNs, this computational mechanism provides two properties that are relevant to navigation. First, membrane integration allows the model to accumulate recent evidence through internal spiking dynamics, rather than responding solely to an instantaneous feature map. Second, the discrete nature of spike-based activity can provide a form of local state stability: when small perturbations do not alter the spiking pattern, their influence on the encoded state remains limited. These characteristics motivate the use of SNNs in navigation policies.

\subsection{Problem Formulation}

We consider embodied navigation as a partially observable Markov decision process. At time step $t$, the agent receives an observation $o_t$ from an egocentric visual stream and maintains an internal state $h_t$ summarizing interaction history. The policy outputs an action $a_t$ that changes both the environment state and the next observation. The task objective is to navigate successfully to a specified goal, such as a target point in PointNav or an instance of a target object category in ObjectNav.

\subsection{SpikingNav}

SpikingNav adopts a minimally disruptive design that preserves the standard embodied RL interface while introducing spiking computation only into the sensing encoder and the policy core.
Fig.~\ref{fig:framework} illustrates the overall pipeline. Given the current visual observation and a target token, SpikingNav first extracts sensing features using the SSE. Specifically, the visual observation is processed by a visual backbone to obtain observation features, while the target token is transformed into a target feature through an embedding layer. The embedding layer maps each target category identifier to a learnable feature representation. For PointNav, the goal input is a continuous goal descriptor provided by the environment and is transformed through a goal encoder. These two representations are then aligned and compressed by a lightweight fusion backbone before being fed into the SPN, which maintains a spike-driven temporal state and produces actor-critic outputs.

\subsubsection{Spiking Sensing Encoder}

As shown in Fig.~\ref{fig:framework}, the SSE consists of three components: a visual backbone for extracting observation features, a target backbone for encoding the goal input, and a lightweight fusion backbone for aligning the joint representation.

Let $x_t \in \mathbb{R}^{H \times W \times C}$ denote the current RGB or RGB-D observation at time step $t$, where $H$ and $W$ are the input spatial dimensions and $C$ depends on the enabled sensor configuration. The visual backbone consists of a spike-based ResNet18-style network~\cite{zhang2026burst} followed by a feature compressor with two convolutional layers. Given $x_t$, the encoder processes the input over $T$ internal spiking steps and produces a visual representation:
\begin{equation}
z_t = E_{\mathrm{vis}}(x_t; T), \qquad z_t \in \mathbb{R}^{H_z \times W_z \times D_z},
\end{equation}
where $H_z$, $W_z$, and $D_z$ denote the spatial height, spatial width, and channel dimension of the visual feature, respectively. Following the original model setting~\cite{zhang2026burst}, we set $T=4$. In our implementation, the feature compressor uses two convolutional layers, with intermediate and output channel dimensions of 128 and $D_z$, respectively. We set $H_z=7$, $W_z=7$, and $D_z=32$.

The target backbone encodes the task-specific goal input $g_t$ into a target representation:
\begin{equation}
% q_t = E_{\mathrm{tar}}(g_t), \qquad q_t \in \mathbb{R}^{H_z \times W_z \times D_q},
e_t = E_{\mathrm{tar}}(g_t), \qquad e_t \in \mathbb{R}^{D_q},
\end{equation}
The target vector is spatially expanded to match the visual feature map, yielding \(q_t \in \mathbb{R}^{H_z \times W_z \times D_q}\),
where $E_{\mathrm{tar}}(\cdot)$ denotes the target embedding layer and $D_q$ is the channel dimension of the target representation. In our implementation, we set $D_q = D_z$. For ObjectNav, $g_t$ corresponds to the target object token; for PointNav, it denotes the goal descriptor provided by the environment.

The visual feature $z_t$ and target feature $q_t$ are concatenated along the channel dimension and then processed by a lightweight fusion module with two convolutional layers followed by a flatten operation:
\begin{equation}
r_t = F_{\mathrm{fus}}\big([z_t; q_t]\big), \qquad r_t \in \mathbb{R}^{D_{\mathrm{in}}},
\end{equation}
where $[z_t; q_t] \in \mathbb{R}^{H_z \times W_z \times (D_z + D_q)}$, and $D_{\mathrm{in}} = H_z W_z D_z$.

\subsubsection{Spiking Policy Network}~\label{sec:spn}
The navigation policy of SpikingNav is realized by an SPN, as illustrated in Fig.~\ref{fig:framework}. 
Let $c_t^{\mathrm{sens}}$ and $c_t^{\mathrm{mem}}$ denote the sensory and recurrent currents, respectively, where $W_r$ and $W_h$ are learnable weight matrices:
\begin{align}
c_t^{\mathrm{sens}} &= W_r r_t, \\
c_t^{\mathrm{mem}} &= W_h u_{t-1}.
\end{align}
The spiking recurrent core then updates its internal state as:
\begin{align}
   v_t &= \lambda u_{t-1} + c_t^{\mathrm{sens}} + c_t^{\mathrm{mem}},  ~\label{eq:vt} \\ 
   s_t &= \mathcal{H}(v_t - \vartheta), \\
u_t &= v_t \odot (1 - s_t). ~\label{eq:ut}
\end{align}
Here, \(v_t\) denotes the pre-reset membrane potential that integrates the sensory embedding \(r_t\) and the previous memory state \(u_{t-1}\), while \(\lambda \in (0,1)\) controls the decay of past memory. The binary spike signal \(s_t\), generated by the Heaviside step function \(\mathcal{H}(\cdot)\) with threshold \(\vartheta\), further serves as a {spike gate} that determines whether the newly accumulated membrane state is preserved or reset. Specifically, when the neuron does not fire (\(s_t=0\)), the gate remains closed and the membrane potential is directly carried forward, yielding \(u_t = v_t\). When the neuron fires (\(s_t=1\)), the gate opens and resets the fired state component through \(u_t = v_t \odot (1 - s_t)\). This mechanism enables event-driven memory regulation: subthreshold information is retained for future computation, while threshold-crossing activations are selectively cleared after spike emission.  During training, the non-differentiable spike function \(\mathcal{H}(\cdot)\) is optimized using surrogate gradients. 

Together, these dynamics form a lightweight memory mechanism with implicit gating. Unlike the Spiking-GRU formulations adopted in~\cite{yang2023memorynav,yang2025hsrl}, which introduce explicit GRU-style update and reset gates into the spiking domain, our design is built on the native integration-and-fire dynamics of spiking neurons. The membrane inertia $\lambda u_{t-1}$ in Eq.~\eqref{eq:vt} provides a gradual update of latent history, while the threshold-triggered spike and reset in Eq.~\eqref{eq:ut} make state changes event-dependent. As a result, transient sensory perturbations that remain below the firing threshold have limited influence on the recurrent state.

We use the post-reset membrane state \(u_t\) as the recurrent policy state \(h_t\), that is, \(h_t := u_t\). The policy and value outputs are then defined as:
\begin{align}
\pi(a_t \mid o_{\le t}) &= \mathrm{softmax}(W_{\pi} h_t + b_{\pi}), ~\label{eq:phi} \\
V_t &= W_v h_t + b_v,
\end{align}
where \(\pi(a_t \mid o_{\le t})\) denotes the action distribution and \(V_t\) is the scalar value estimate at time step \(t\).

\subsubsection{Training Method}

SpikingNav is trained in the same embodied RL pipeline based on AllenAct~\cite{allenact2020}, from which we also adopt its corresponding hyperparameter settings, with the value loss coefficient $\alpha$ set to 0.5 and the entropy coefficient $\beta$ set to 0.01. The model is optimized with a PPO-style actor-critic objective:
\begin{equation}
\mathcal{L} = \mathcal{L}_{\text{policy}} + \alpha \mathcal{L}_{\text{value}} - \beta \,\mathrm{Ent}(\pi),
\end{equation}
where \(\mathcal{L}_{\text{policy}}\) and \(\mathcal{L}_{\text{value}}\) denote the policy and value losses, respectively, \(\mathrm{Ent}(\pi)\) is the policy entropy regularizer, and \(\alpha\) and \(\beta\) are weighting coefficients.

For the policy term, we use the PPO clipped surrogate objective:
\begin{equation}
\mathcal{L}_{\text{policy}}
= - \mathbb{E}_t \Big[
\min \big(
\eta_t(\theta)\hat{A}_t,\,
\mathrm{clip}(\eta_t(\theta),1-\epsilon,1+\epsilon)\hat{A}_t
\big)
\Big],
\end{equation}
where
\begin{equation}
\eta_t(\theta)=\frac{\pi_{\theta}(a_t \mid o_{\le t})}{\pi_{\theta_{\mathrm{old}}}(a_t \mid o_{\le t})}.
\end{equation}
Here, \(\pi_{\theta}\) and \(\pi_{\theta_{\mathrm{old}}}\) denote the current and old policies, respectively, \(\epsilon\) is the clipping threshold, and \(\hat{A}_t = \hat{R}_t - V_t\) is the advantage estimate at time step \(t\), where \(V_t\) is the critic output defined above and \(\hat{R}_t\) is the target return.

The value loss is defined as:
\begin{equation}
\mathcal{L}_{\text{value}} = \mathbb{E}_t \left[(V_t - \hat{R}_t)^2\right].
\end{equation}

\subsection{A Temporal Robustness Perspective}

We analyze the local temporal stability of the SPN dynamics introduced in Sec.~\ref{sec:spn}. This analysis provides a sufficient condition under which bounded perturbations in sensory features preserve the spike pattern of the recurrent policy core and induce a bounded change in the actor logits.

Let \((u_t,r_t,v_t,s_t)\) and \((\hat u_t,\hat r_t,\hat v_t,\hat s_t)\) denote the clean and perturbed SPN variables at navigation step \(t\), respectively. We define:
\begin{equation}
\begin{aligned}
\Delta u_t &= \hat u_t-u_t, \qquad
\Delta v_t = \hat v_t-v_t,\\
\Delta r_t &= \hat r_t-r_t, \qquad
d_t = W_r \Delta r_t.
\end{aligned}
\label{eq:aligned_delta}
\end{equation}
We follow the notation of Sec.~\ref{sec:spn}, where \((u_t,r_t,v_t,s_t)\) denote the SPN variables at navigation step \(t\). For a perturbed sensory sequence, we use hatted variables \((\hat u_t,\hat r_t,\hat v_t,\hat s_t)\) to denote the corresponding perturbed trajectory. 

\textbf{Accumulated perturbation budget.} Using the compact notation:
\begin{equation}
\mathcal{M}=\lambda I+W_h ,
\end{equation}
where $I$ denotes the identity matrix, the pre-reset membrane update in Eq.~\eqref{eq:vt} can be written as
\begin{equation}
v_t=\mathcal{M}u_{t-1}+W_r r_t .
\label{eq:compact_vt}
\end{equation}
Applying the same update to the perturbed trajectory and subtracting Eq.~\eqref{eq:compact_vt}, we obtain:
\begin{equation}
\Delta v_t
=
\mathcal{M}\Delta u_{t-1}
+
d_t .
\label{eq:delta_vt}
\end{equation}
Therefore,
\begin{equation}
\|\Delta v_t\|_\infty
\le
\|\mathcal{M}\|_\infty
\|\Delta u_{t-1}\|_\infty
+
\|d_t\|_\infty .
\label{eq:one_step_bound}
\end{equation}

Over a finite interval \(\{\tau,\ldots,t\}\), Eq.~\eqref{eq:one_step_bound} induces a worst-case accumulated perturbation budget. For any \(j\in\{\tau,\ldots,t\}\), we define:
\begin{equation}
B_{\tau,j}
=
\|\mathcal{M}\|_{\infty}^{\,j-\tau+1}
\|\Delta u_{\tau-1}\|_{\infty}
+
\sum_{k=\tau}^{j}
\|\mathcal{M}\|_{\infty}^{\,j-k}
\|d_k\|_{\infty}.
\label{eq:perturbation_budget}
\end{equation}
The first term describes the initial membrane perturbation propagated to step \(j\), while the summation term accumulates the sensory-current perturbations injected from step \(\tau\) to step \(j\). Thus, \(B_{\tau,j}\) upper-bounds the worst-case membrane perturbation accumulated over the interval.

\textbf{Spike margin.} Since the spike output is determined element-wise by the sign of \(v_t-\vartheta\), we define the local spike margin for neuron \(i\) as:
\begin{equation}
m_{t,i}=|v_{t,i}-\vartheta|,
\qquad
\bar m_t=\min_i m_{t,i}.
\label{eq:aligned_margin}
\end{equation}
Here, \(i\) indexes the neuron dimension, and \(\bar m_t\) is the minimum distance from the clean membrane potential to the shared firing threshold.

% The following propositions state the local robustness mechanism. \textbf{Proposition~1} compares the accumulated perturbation budget in Eq.~\eqref{eq:perturbation_budget} with the spike margin in Eq.~\eqref{eq:aligned_margin} to obtain a spike-pattern preservation condition. \textbf{Proposition~2} then propagates the resulting state bound to the actor logits.

The following proposition states the local robustness mechanism. It compares the accumulated perturbation budget in Eq.~\eqref{eq:perturbation_budget} with the spike margin in Eq.~\eqref{eq:aligned_margin}. When the accumulated perturbation remains smaller than the margin to the firing threshold, the spike pattern is preserved over the considered interval. We then discuss its implication for the actor logits.

\begin{proposition}[Spike-pattern preservation]
\label{pro1}
Consider a finite interval \(\{\tau,\ldots,t\}\), and let \(B_{\tau,j}\) be the perturbation budget defined in Eq.~\eqref{eq:perturbation_budget}. If:
\begin{equation}
B_{\tau,j} < \bar m_j,
\qquad
\forall j\in\{\tau,\ldots,t\},
\label{eq:spike_preservation_condition}
\end{equation}
then the perturbed trajectory preserves the same spike pattern as the clean trajectory over this interval, i.e.,
\begin{equation}
\hat s_j=s_j,
\qquad
\forall j\in\{\tau,\ldots,t\}.
\end{equation}
\end{proposition}

\begin{proof}
From Eq.~\eqref{eq:perturbation_budget}, the membrane perturbation at step \(j\) is bounded by:
\begin{equation}
\|\Delta v_j\|_\infty \le B_{\tau,j}.
\end{equation}
Under the condition \(B_{\tau,j}<\bar m_j\), we have:
\begin{equation}
\|\Delta v_j\|_\infty < \bar m_j
\le |v_{j,i}-\vartheta|,
\qquad \forall i .
\end{equation}
Therefore, for each neuron \(i\), the perturbation cannot move the membrane potential across the firing threshold \(\vartheta\). Hence,
\begin{equation}
\mathrm{sign}(\hat v_{j,i}-\vartheta)
=
\mathrm{sign}(v_{j,i}-\vartheta),
\qquad \forall i .
\end{equation}
Since the spike output is determined by this threshold comparison, we obtain:
\begin{equation}
\hat s_j=s_j,
\qquad
\forall j\in\{\tau,\ldots,t\}.
\end{equation}
This completes the proof.
\end{proof}

% \textbf{Implication for policy-logit stability.} Proposition~\ref{pro1} implies that, within the interval \(\{\tau,\ldots,t\}\), the clean and perturbed trajectories share the same spike pattern. Therefore, the reset operation does not introduce additional spike-state mismatch. For hard reset, the perturbation is either inherited from the pre-reset membrane potential or set to zero. For soft reset, the identical threshold subtraction cancels out. Thus, the post-reset perturbation is bounded by the pre-reset perturbation: \begin{equation} \|\Delta u_j\|_\infty \le \|\Delta v_j\|_\infty \le B_{\tau,j}, \qquad \forall j\in\{\tau,\ldots,t\}. \label{eq:state_bound} \end{equation} Since the recurrent policy state is \(h_j:=u_j\), we have \begin{equation} \|\Delta h_j\|_\infty = \|\Delta u_j\|_\infty \le B_{\tau,j}. \end{equation} For the policy logits \begin{equation} \ell_j = W_{\pi}h_j+b_{\pi}, \end{equation} the perturbation satisfies \begin{equation} \|\hat \ell_j-\ell_j\|_\infty = \|W_{\pi}(\hat h_j-h_j)\|_\infty \le \|W_{\pi}\|_\infty \|\Delta h_j\|_\infty \le \|W_{\pi}\|_\infty B_{\tau,j}. \label{eq:policy_logit_bound} \end{equation} This shows that spike-pattern preservation further limits the actor-logit variation through the accumulated perturbation budget \(B_{\tau,j}\).

\textbf{Implication for actor-logit variation.}
The spike-pattern preservation condition also bounds the variation of the recurrent policy state. Since the clean and perturbed trajectories share the same spike pattern over \(\{\tau,\ldots,t\}\), the reset operation is applied in the same way to both trajectories. Hence, it does not amplify the membrane perturbation, yielding:
\begin{equation}
\begin{aligned}
\|\Delta h_j\|_\infty
&=
\|\Delta u_j\|_\infty  \\
&\le
\|\Delta v_j\|_\infty
\le
B_{\tau,j},
\qquad
\forall j\in\{\tau,\ldots,t\}.
\end{aligned}
\label{eq:state_bound}
\end{equation}
For the policy logits:
\begin{equation}
\ell_j = W_{\pi}h_j+b_{\pi},
\end{equation}
we therefore obtain:
\begin{equation}
\begin{aligned}
\|\hat \ell_j-\ell_j\|_\infty
&=
\|W_{\pi}(\hat h_j-h_j)\|_\infty  \\
&\le
\|W_{\pi}\|_\infty
\|\Delta h_j\|_\infty  \\
&\le
\|W_{\pi}\|_\infty B_{\tau,j}.
\end{aligned}
\label{eq:policy_logit_bound}
\end{equation}
Thus, once the perturbation budget remains below the spike margin, the actor-logit variation is controlled by the same accumulated budget \(B_{\tau,j}\), scaled by the policy-head norm.

\section{Experiments}

\subsection{Dataset and Evaluation Metric}

The experiments are conducted in RoboTHOR~\cite{deitke2020robothor}. We consider two tasks, namely {PointNav} and {ObjectNav}. In PointNav, the agent is required to reach a target position specified relative to its current pose, which mainly evaluates goal-conditioned navigation with explicit geometric guidance. In ObjectNav, the agent must navigate to an instance of a target object category from egocentric observations, making the problem more challenging due to the additional need for semantic grounding and target search. Detailed experimental setups, evaluation protocols, and full training configurations are provided in Supplementary Materials S1.

We report {Success Rate (SR)} and {Success weighted by Path Length (SPL)} as the main evaluation metrics. SR measures the fraction of successful episodes, whereas SPL further accounts for navigation efficiency by comparing the executed trajectory with the shortest feasible path. These two metrics are reported for both clean and perturbed settings so that robustness can be assessed in terms of both task completion and path quality.

Robustness is evaluated under the visual corruptions in RobustNav~\cite{robustnav}. The visual corruptions include low lighting, motion blur, camera crack, defocus blur, speckle noise, reduced field of view (lower FOV), and spatter, which degrade the quality of egocentric observations. \revise{These corruptions emulate common physical-system disturbances, testing whether the policy remains stable when the physical sensing is degraded.}

\subsection{Navigation Performance, Model Scale, and Efficiency}

Table~\ref{tab:full_results} compares the performance of SpikingNav and ANN navigation methods on different navigation tasks. On PointNav, SpikingNav can match strong ANN performance on this standard goal-conditioned navigation task.
For the more challenging {ObjectNav} task, recent high-performing ANN-based ObjectNav systems often rely on pretrained vision-language models such as CLIP~\cite{khandelwalSimpleEffectiveCLIP2022}, or even external large language models~\cite{yin2024sg,wen2025zero}, to enhance semantic perception and decision making. Although effective, these designs are considerably less suitable for edge deployment due to their model size and computational overhead. In contrast, our comparison focuses on compact models under a matched parameter budget. Under this setting, SpikingNav achieves better ObjectNav performance than its ANN counterpart RobustNav~\cite{robustnav} and SNN counterpart LIF+GRU, suggesting that spiking policies offer a more favorable solution when strong embodied navigation performance must be obtained with small-scale models.

To enable a fair comparison, we re-implement the ANN pipeline of RobustNav~\cite{robustnav} using the same data processing, training protocol, and evaluation setting as SpikingNav, and denote this model as {ANNNav}. The parameter count, computational cost (FLOPs), and task performance of ANNNav and SpikingNav are reported in Table~\ref{tab:complexity}. Under this matched setting, SpikingNav is competitive on PointNav and outperforms ANNNav on ObjectNav while using fewer parameters and substantially fewer FLOPs, indicating that SNNs are a promising solution when strong embodied navigation performance must be achieved with small models.

\begin{figure*}[t]
  \centering
  \includegraphics[width=0.9\linewidth]{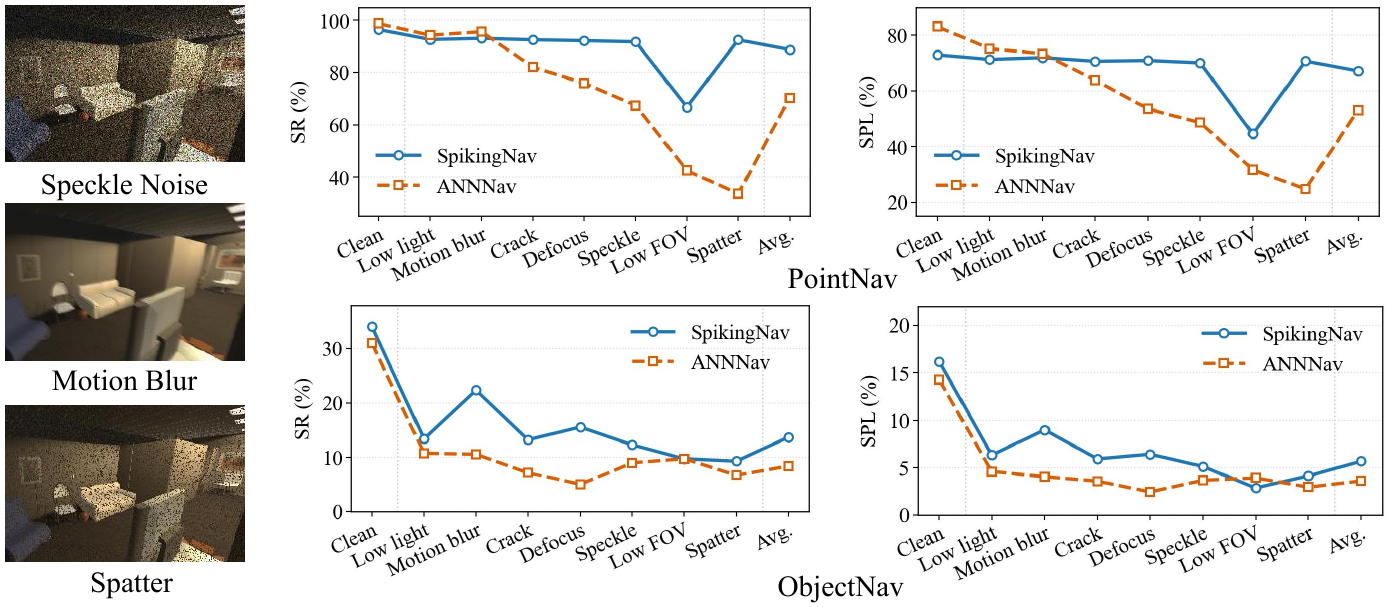}
  \caption{Robustness comparison under visual corruptions on PointNav and ObjectNav. The left column presents representative corrupted observations, including speckle noise, motion blur, and spatter. The right panels compare the SR and SPL of SpikingNav and ANNNav across different visual corruptions. ``Avg.'' denotes the average over the seven corrupted settings, excluding the clean setting.}
  \label{fig:fig2}
\end{figure*}

%====
\begin{table}[t]
  \centering
  \caption{PointNav and ObjectNav results on RoboTHOR.}
  \label{tab:full_results}
  \setlength{\tabcolsep}{2pt}
  \renewcommand{\arraystretch}{1.08}
  \begin{tabularx}{\columnwidth}{@{}>{\raggedright\arraybackslash}Xccc@{}}
    \toprule
    Method & Params & SR(\%) & SPL(\%) \\
    \midrule
    \multicolumn{4}{l}{\textit{PointNav}} \\
    \midrule
    IL RGB & 14M & 61.0 & 52.0 \\
    DD-PPO Depth \cite{allenact2020} & 14M & 93.0 & 88.0 \\
    RobustNav-RGB DD-PPO \cite{chattopadhyayRobustNavBenchmarkingRobustness2021}& 14M & 98.8 & 83.1 \\
    RobustNav-RGBD DD-PPO \cite{chattopadhyayRobustNavBenchmarkingRobustness2021} & 14M & 98.5 & 84.6 \\
    % ENTL RGB \cite{kotarENTLEmbodiedNavigation2023} & 30M & -- & 92.0 & 87.0 \\
    ENTL RGB+CLIP \cite{kotarENTLEmbodiedNavigation2023} & $>$69M & 95.0 & 91.0 \\
    \textbf{SpikingNav (Ours)} & \textbf{12.1M} & \textbf{96.5} & \textbf{72.9} \\ % 96.5 & 72.9 % & 97.6 & 76.2
    \midrule
    \multicolumn{4}{l}{\textit{ObjectNav}} \\
    \midrule
    LIF+GRU & 12.4M & 23.4 & 10.2 \\
    Action Boost \cite{batraRearrangementChallengeEmbodied2020}& / & 28.0 & 12.0 \\
    Dreamer \cite{hafnerDreamControlLearning2020}& / & 31.0 & 6.0 \\
    uORF \cite{yuUnsupervisedDiscoveryObject2022}& / & 31.3 & 14.6 \\
    RGB+D ResNet18 \cite{allenact2020}& 14M & 26.0 & 11.0 \\
    RobustNav-RGB DD-PPO \cite{chattopadhyayRobustNavBenchmarkingRobustness2021}& 14M & 31.1 & 14.3 \\
    % ENTL \cite{kotarENTLEmbodiedNavigation2023}& 10M & -- & 35.0 & 6.0 \\
    % NRC \cite{wallingfordNeuralRadianceField2023}& $\sim$20M & -- & 50.1 & 23.9 \\
    \midrule
    ENTL \cite{kotarENTLEmbodiedNavigation2023}& 96M & 46.0 & 16.0 \\
    EmbCLIP \cite{khandelwalSimpleEffectiveCLIP2022}& $\sim$40M & 47.0 & 20.0 \\
    SG-Nav-LLaMA \cite{yin2024sg}& $\sim$14B & 47.3 & 23.7 \\
    SG-Nav-GPT \cite{yin2024sg}& $>$180B & 47.5 & 24.0 \\
    VLTNet \cite{wen2025zero}& $>$100B & 33.2 & 17.1 \\
    \midrule
    \textbf{SpikingNav (Ours)} & \textbf{12.1M} & \textbf{34.1} & \textbf{16.2} \\
    \bottomrule
  \end{tabularx}
\end{table}

\begin{table}[t]
\caption{Model scale and clean performance. {PointNav and ObjectNav entries are reported as SR / SPL (\%). FLOPs denote the per-step forward computation measured under the same input resolution.}}
\label{tab:complexity}
\centering
\small
\begin{tabular}{lcccc}
\toprule
Method & Params & FLOPs & PointNav & ObjectNav \\
\midrule
ANNNav & 14M & 4.21G & 98.21 / 82.13 & 31.05 / 14.26 \\
SpikingNav & 12.1M & 0.97G & 96.54 / 72.93 & 34.12 / 16.20 \\
\bottomrule
\end{tabular}
\end{table}

\vspace{-0.3em}
\subsection{Navigation Robustness}

Fig.~\ref{fig:fig2} reports PointNav performance under clean conditions and a range of visual corruptions. Under clean conditions, ANNNav achieves the highest SR and SPL. However, under several visual corruptions, SpikingNav outperforms ANNNav by a substantial margin. These results suggest that SpikingNav achieves near-parity with ANNNav under clean conditions, while exhibiting stronger robustness.

Compared with PointNav, ObjectNav places substantially heavier demands on semantic perception and long-horizon state maintenance, since the agent must not only navigate but also identify and continuously pursue a target object from egocentric observations. As shown in {Fig.~\ref{fig:fig2}}, SpikingNav achieves a higher SR and SPL than ANNNav. Under most visual perturbations, this advantage is preserved or further enlarged, with especially clear gains under motion blur, camera crack, and defocus blur. Only under lower FOV does ANNNav remain slightly stronger. 

We provide qualitative navigation examples in Supplementary Material S2 to further illustrate how SpikingNav behaves under visual conditions and goal-driven navigation scenarios.

\vspace{-0.3em}
\subsection{Ablation Study}

\begin{table}[t]
\caption{Module-replacement ablation on ObjectNav. Starting from the ANNNav baseline, we replace its ANN perception module with the SSE, and its ANN policy module with the SPN. Results are reported as SR / SPL (\%). Robustness denotes the average performance over the visual corruptions.}
\label{tab:ablation}
\centering
\setlength{\tabcolsep}{3.0pt}
\renewcommand{\arraystretch}{1.08}
\begin{tabular}{lccccc}
\toprule
Variant & SSE & SPN & Clean & Robustness \\
\midrule
ANNNav baseline & -- & -- & 31.05 / 14.26 & 8.45 / 3.59 \\
SSE only & \checkmark & -- & 30.70 / 15.30 & 9.58 / 4.74 \\
SPN only & -- & \checkmark & 31.00 / 12.60 & 10.20 / 4.25 \\
\textbf{SpikingNav} & \checkmark & \checkmark & \textbf{34.12 / 16.20} & \textbf{13.71 / 5.67} \\
\bottomrule
\end{tabular}
\end{table}

To clarify which components contribute to the robustness gains, we add a controlled module ablation around the two design choices in SpikingNav: SSE and SPN.

\textbf{SSE.}
As shown in Table~\ref{tab:ablation}, the SSE-only variant replaces the ANN perception module in ANNNav with the SSE while keeping the ANN policy module unchanged. Relative to ANNNav, this replacement slightly lowers clean SR (30.70\% vs. 31.05\%), but improves clean SPL (15.30\% vs. 14.26\%) and average robustness under visual corruptions ({9.58\%/4.74\%} vs. 8.45\%/3.59\% in SR/SPL). This suggests that replacing the perception front-end with SSE mainly improves the stability of visual feature extraction, leading to more robust task-conditioned representations under corrupted observations even when the downstream policy remains ANN-based.

\textbf{SPN.}
Table~\ref{tab:ablation} also shows that the SPN-only variant keeps the ANN perception module and replaces only the ANN policy module with the SPN. Under this setting, clean SR remains nearly unchanged relative to ANNNav (31.00\% vs. 31.05\%), while robustness improves to 10.20\%/4.25\% in SR/SPL. This indicates that replacing the policy core with SPN can improve robustness even without modifying the visual front-end. At the same time, clean SPL decreases to 12.60\%, suggesting that the robustness gain introduced by spike-based recurrent policy dynamics may be accompanied by some loss in path efficiency under clean conditions.

\textbf{Full model.}
Finally, Table~\ref{tab:ablation} shows that replacing both ANN modules with their spike-based counterparts yields the full SpikingNav model, which achieves the best robustness result in the ablation study, reaching {13.71\%/5.67\%} in SR/SPL, compared with 8.45\%/3.59\% for ANNNav. It also attains the highest clean SR ({34.12\%}) with SPL ({16.20\%}). Overall, the ablation results suggest that the two replacements play complementary roles: SSE improves the robustness of the perceptual representation, whereas SPN improves the robustness of temporal policy updating. Their combination yields the strongest overall gain under corruptions.
% \begin{figure}[t]
%   \centering
%   \includegraphics[width=0.95\linewidth]{figures/visual_2_snn.pdf}
%   \caption{Qualitative results of SpikingNav on ObjectNav under speckle-noise corruption. 
% The upper row shows the egocentric observation and the top-down trajectory, with the blue curve denoting the agent's navigation path. 
% The lower rows show consecutive corrupted observations from the same episode.}
% \label{fig:qualitative_speckle}
%   \label{fig:qualitative}
% \end{figure}

% \subsection{Qualitative Results}

% We further provide an episode-level visualization to complement the quantitative robustness results. As shown in Fig.~\ref{fig:qualitative}, SpikingNav is evaluated on ObjectNav under speckle noise corruption. The top row shows the clean egocentric observation and the corresponding top-down trajectory, where the blue curve denotes the agent trajectory. The bottom row presents several corrupted egocentric observations sampled along the trajectory.
% Although the egocentric observations are severely degraded by speckle noise, SpikingNav still produces a continuous trajectory and maintains stable navigation behavior across consecutive steps. This qualitative result supports the quantitative robustness analysis by showing that the proposed spiking navigation policy can preserve coherent behavior under noisy visual inputs.

\vspace{-0.3em}
\subsection{Neuromorphic Hardware Deployment}
\label{sec:Hardware}

To examine the hardware deployability of SpikingNav \revise{as a physical computing system}, we mapped its SSE onto Thruster--V2%
\footnote{The hardware implementation work of Thruster--V2 has not been published yet.\label{foot:gaban_name}}. 
It is a multi-core neuromorphic processor derived from the GaBAN architecture~\cite{chen2022gaban} and has been taped-out. 
Following the memory prefetching and buffering design of GaBAN, Thruster--V2 supports Buffets~\cite{pellauer2019buffets} prefetching and asynchronous buffering, which helps reduce the impact of memory-access latency in SNN workloads. The deployed SSE includes the spike-based visual backbone, target-token encoding, and goal-conditioned feature fusion, which together form the sensing \revise{front} end of SpikingNav. 
The deployment was executed on one GaBAN processing core under a \(5\,\mathrm{V}\) supply voltage. 
Cycle counts were converted using a \(625\,\mathrm{MHz}\) cycle-counter frequency. Each run used an input batch with the same scale as that used in the simulator evaluation, containing \(128\) sampled inputs. Each input consists of one visual observation and the corresponding target category token. The reported per-sample cycles and energy were averaged from the batch-level measurements. 
We repeated the measurement five times to reduce the influence of operational variation. The deployed spiking sensing module required \((2.133 \pm 0.048)\times10^{9}\) cycles per sample and the measured dynamic energy consumption was \(3.92 \pm 0.14\) J per sample. These results verify that the sensing and target-fusion computation of SpikingNav can be mapped to and executed on a physical neuromorphic hardware platform, providing an initial hardware-level validation of the proposed framework. The current implementation prioritizes functional compatibility and programmability, while further hardware-aware sparse execution and event-driven optimization of SpikingNav may reduce the deployment cost.

\begin{table}[t]
    \centering
    \caption{On-chip execution statistics of the deployed spiking sensing module on the Thruster--V2 neuromorphic chip.}
    \label{tab:chip_deployment}
    \setlength{\tabcolsep}{4.5pt}
    \renewcommand{\arraystretch}{1.12}
    \begin{tabular}{lc}
        \toprule
        \textbf{Deployment setting or metric} & \textbf{Result} \\
        \midrule
        Supply voltage & \(5\) V \\
        Chip frequency & \(625\) MHz \\
        Cycles per sample & \(2.133 \pm 0.048\) G~cycles \\
        Dynamic current & \(229.5 \pm 4.7\) mA \\
        Dynamic energy per sample & \(3.92 \pm 0.14\) J \\
        \bottomrule
    \end{tabular}
\end{table}

\begin{figure}[t]
    \centering
    \includegraphics[width=0.95\linewidth]{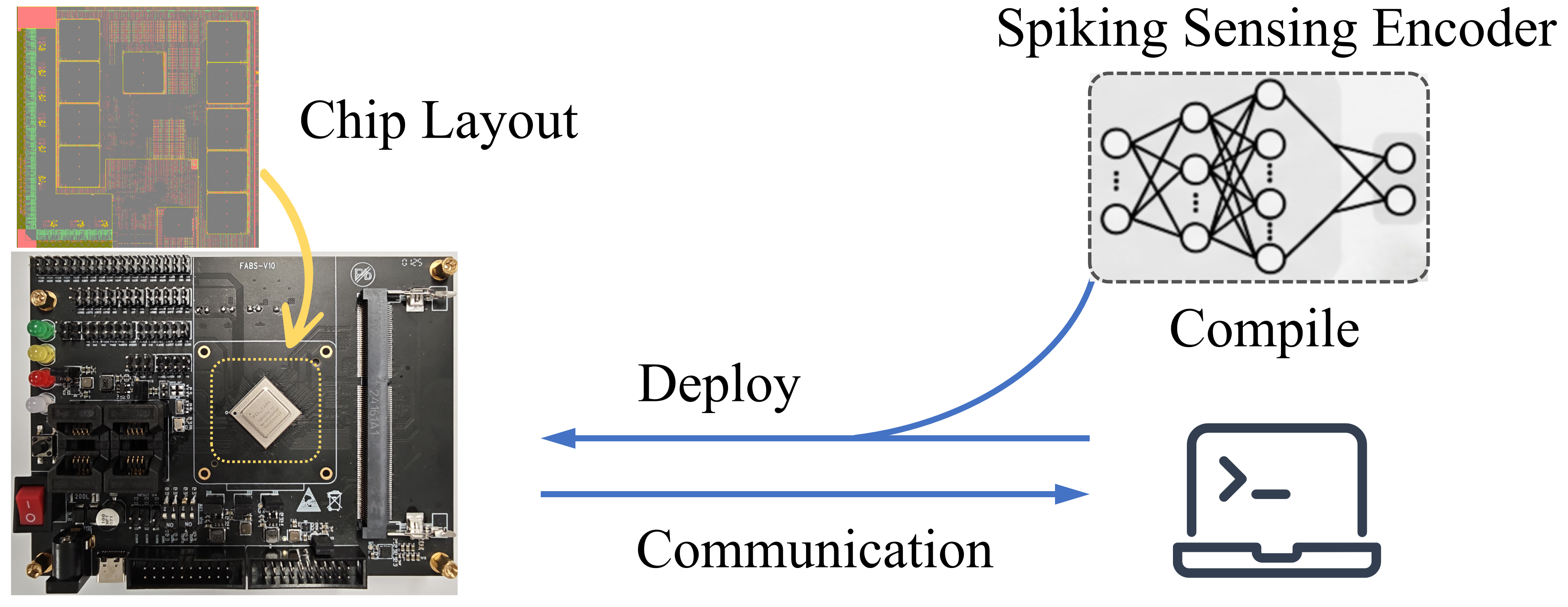}
    \caption{Deployment process on the Thruster--V2\textsuperscript{\ref{foot:gaban_name}}. The right subfigure shows the layout of the taped-out Thruster--V2 multi-core neuromorphic processor, \revise{which provides the physical hardware basis for the SpikingNav}.}
    \label{fig:placeholder}
\end{figure}

% \vspace{-0.4em}
\section{Analysis and Discussion}
In static vision tasks, such as image recognition on ImageNet~\cite{imagenet}, SNN backbones often struggle to surpass the performance of their ANN counterparts. Furthermore, while SNNs demonstrate potential advantages in adversarial robustness, recent studies suggest that their resilience to real-world image corruptions remains inferior to that of corresponding ANNs~\cite{mpsn, zhang2026burst}. To provide a comprehensive evaluation of both clean accuracy and robustness, we compared the SNN backbone used in SpikingNav with the ResNet backbone from ANNNav using the ImageNet-C dataset~\cite{hendrycks2019benchmarking}. This benchmark encompasses a wide range of corruptions, such as noise, blur, and lighting variations, mirroring the disturbances found in the RobustNav benchmark.

As shown in Table~\ref{tab:backbone_recognition}, the SNN encoder achieves nearly the same clean recognition accuracy as the ANN encoder, but obtains a lower corruption average (31.32\% vs. 34.41\%). In contrast, the navigation results show a different trend. SpikingNav improves clean ObjectNav SR from 31.05\% to 34.12\%, and increases corrupted ObjectNav SR from 8.45\% to 13.71\%. This indicates that embodied navigation does not simply inherit the single-frame recognition ranking of the visual backbone.

To quantify this difference, we compute a retention ratio $R_{\mathrm{ret}}=\frac{P_{\mathrm{corrupt}}}{P_{\mathrm{clean}}}$,
where $P_{\mathrm{clean}}$ and $P_{\mathrm{corrupt}}$ denote the clean and corruption-averaged performance under the same evaluation protocol. For recognition, $P$ denotes classification accuracy; for navigation, $P$ denotes ObjectNav SR. As summarized in Fig.~\ref{fig:recognition_navigation_retention}, the ANN encoder retains 49.33\% of its clean recognition accuracy under corruptions, while the SNN encoder retains 44.82\%. In contrast, ANNNav retains only 27.21\% of its clean ObjectNav SR, whereas SpikingNav retains 40.18\%. Therefore, although the SNN encoder is not more robust as an isolated frame classifier, the full spiking navigation system preserves a much larger fraction of task success under corrupted observations.

\begin{table}[t]
\caption{Comparison between backbone recognition and embodied navigation under clean and corrupted visual observations. Recognition performance is evaluated using the ANN or SNN visual encoder as a static image classifier.}
\label{tab:backbone_recognition}
\centering
\setlength{\tabcolsep}{3.5pt}
\renewcommand{\arraystretch}{1.08}
\begin{tabular}{lcccc}
\toprule
\multirow{2}{*}{Model} & \multicolumn{2}{c}{Recognition} & \multicolumn{2}{c}{Navigation} \\
\cmidrule(lr){2-3}\cmidrule(lr){4-5}
& Clean Acc. & Corrupt Avg. & Clean SR & Corrupt SR \\
\midrule
ANNNav & 69.76\% & 34.41\% & 31.05\% & 8.45\% \\
SpikingNav & 69.89\% & 31.32\% & 34.12\% & 13.71\% \\
\bottomrule
\end{tabular}
\end{table}

\begin{figure}[t]
\centering
\definecolor{ieeeblue}{RGB}{0,114,178}
\definecolor{ieeeorange}{RGB}{230,159,0}
\begin{tikzpicture}
\begin{axis}[
    width=0.98\columnwidth,
    height=4.0cm,
    ylabel={Retention ratio (\%)},
    symbolic x coords={Recognition,Navigation},
    xtick=data,
    ymin=0,
    ymax=60,
    ytick={0,15,30,45,60},
    tick align=inside,
    axis line style={black, line width=0.45pt},
    tick style={black, line width=0.45pt},
    yticklabel style={font=\scriptsize},
    xticklabel style={font=\scriptsize},
    label style={font=\scriptsize},
    legend style={
        font=\scriptsize,
        at={(0.5,1.03)},
        anchor=south,
        legend columns=2,
        draw=none,
        fill=none,
        /tikz/every even column/.append style={column sep=0.15cm}
    },
    grid=major,
    grid style={gray!25, line width=0.25pt},
    major grid style={gray!25, line width=0.25pt},
]
\addplot[
    color=ieeeblue,
    line width=0.9pt,
    mark=*,
    mark size=1.8pt
] coordinates {(Recognition,49.33) (Navigation,27.21)};

\addplot[
    color=ieeeorange,
    line width=0.9pt,
    mark=square*,
    mark size=1.8pt
] coordinates {(Recognition,44.82) (Navigation,40.18)};

\legend{ANNNav, SpikingNav}
\end{axis}
\end{tikzpicture}
\caption{Retention ratios of ANNNav and SpikingNav under visual corruptions. 
% The blue line with circle markers denotes ANNNav, and the orange line with square markers denotes SpikingNav.
Recognition denotes the retention ratio of ImageNet-C classification accuracy, while Navigation denotes the retention ratio of ObjectNav success rate.}
\label{fig:recognition_navigation_retention}
\end{figure}
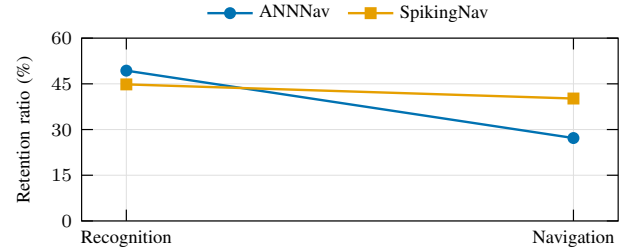

This discrepancy reflects the difference between static recognition and embodied navigation. Static recognition evaluates whether a corrupted frame can still map to the correct category. Embodied navigation instead evaluates whether an agent can maintain effective action decisions while observations, internal states, and environmental feedback evolve over time. Thus, a corrupted observation is not merely a perception error; it becomes a perturbation to a closed-loop decision process. Conversely, a moderately degraded visual feature does not necessarily cause task failure if the policy state can absorb the perturbation and preserve action-relevant information.

From this perspective, the advantage of SpikingNav does not lie in producing universally more robust static visual features. Rather, it appears when spike-based perception is coupled with spike-based policy dynamics. The SSE converts visual evidence into a compact task-conditioned representation, while the SPN further regulates this representation through membrane accumulation, thresholding, and spike-triggered reset. Small feature fluctuations can be integrated without immediately changing the spike pattern, and transient responses can be suppressed after firing. This finding aligns with Proposition~\ref{pro1}, as the spike pattern remains intact when the perturbation-induced membrane change falls below the local spike margin. Therefore, the robustness gain mainly arises from temporal state evolution rather than isolated visual recognition.

The module ablation further supports this interpretation. The SSE-only variant improves corrupted ObjectNav SR from 8.45\% to 9.58\%, indicating that spike-based perception provides a more stable task-conditioned interface. The SPN-only variant increases corrupted SR to 10.20\%, showing that spike-based recurrent policy dynamics can improve robustness even with an ANN visual encoder. The full model achieves 13.71\%, outperforming either partial replacement. This complementarity suggests that robust embodied behavior requires both a corruption-tolerant perceptual interface and a policy state that can temporally regulate noisy evidence.

These findings clarify why embodied navigation provides a meaningful setting for evaluating SNNs. 
The benefit of SpikingNav is not fully reflected by static recognition accuracy, but emerges when spiking perception and spiking policy dynamics are integrated in a closed-loop agent. 
By accumulating, gating, and resetting sensory evidence over time, spiking dynamics provide a task-relevant decision substrate for embodied agents.

% Beyond simulation-based evaluation, we further conduct a preliminary neuromorphic hardware deployment to examine the practical executability of the proposed spike-driven sensing design. 
% As detailed in the supplementary material S1, we deploy the spiking sensing module, including the visual perception backbone and goal-conditioned feature fusion, on GaBAN, a programmable FPGA-based vector neuro-processor for SNN workloads~\cite{chen2022gaban}. 
% This deployment is intended as a module-level hardware feasibility study. The measured results show that the sensing front end of SpikingNav can be mapped to and executed on physical neuromorphic hardware, providing an initial step toward practical hardware deployment.

% \vspace{-0.3em}
\subsection{Limitations and Future Work}

The present study demonstrates the feasibility of incorporating SNNs into a modern embodied navigation pipeline and highlights their performance and robustness under visual corruptions. Several limitations remain for future research.

First, our evaluation is mainly conducted in high-fidelity simulation environments. Although these simulations include realistic sensor noise, the current hardware validation only covers the spiking sensing module and does not yet include the complete closed-loop navigation agent. Future work will extend SpikingNav toward full neuromorphic deployment and hardware-aware sparse execution, enabling a more comprehensive evaluation of practical latency and energy efficiency in embodied systems.
% Although these simulations include realistic sensor noise, SpikingNav has not yet been deployed on physical neuromorphic hardware. 
% In addition, while low-power and event-driven execution are important motivations for SNNs~\cite{roy2019towards}, this work does not provide a dedicated hardware-level analysis of energy efficiency or latency on platforms such as Loihi or Tianjic~\cite{davies2018loihi,pei2019towards}. Future work will extend SpikingNav to neuromorphic deployment and evaluate its practical efficiency in embodied systems.

Second, this work focuses on compact models designed for edge-oriented agents. This setting is aligned with our target scenarios, but it also limits the generality of our conclusions to larger architectures. It remains unclear whether the benefits of discrete activations and neuron-level temporal dynamics persist in large-scale Vision-Language-Action models or more complex long-horizon multimodal tasks. Exploring spike-based navigation in larger embodied intelligence systems is therefore an important future direction.

\section{Conclusion}

This paper presented SpikingNav, a spiking embodied navigation framework that extends SNNs from perceptual encoding to navigation policy learning in visually rich indoor environments \revise{and targets robust navigation in cyber-physical systems}. By integrating an SSE and an SPN into a standard embodied RL pipeline, SpikingNav enables temporally grounded state construction and action generation with spike-based dynamics. Experimental results on PointNav and ObjectNav showed that SpikingNav achieves competitive clean-task performance while exhibiting stronger robustness under visual corruptions, with fewer parameters and lower per-step computation than the ANN counterpart. \revise{The deployment of the SSE on Thruster--V2 further shows that the sensing component can run on a real neuromorphic chip, linking the cyber navigation algorithm to a physical hardware substrate.} These findings suggest that the practical value of SNNs in embodied navigation lies in supporting compact and disturbance-tolerant sensing and policy modules for resource-constrained embodied systems.

% \section*{Acknowledgment}

% This work was partially supported by National Distinguished Young Scholars (62325603), National Natural Science Foundation of China (62236009, U22A20103),  CAS Project for Young Scientists in Basic Research (YSBR-116), Beijing Science and Technology Plan (Z241100004224011).

% \section*{Data Availability}

% The dataset and code will be released upon publication.
% \vspace{-0.3em}
\bibliographystyle{IEEEtran}
\bibliography{references}

\end{document}